\documentclass[letterpaper, 10 pt, conference]{ieeeconf}  

\IEEEoverridecommandlockouts                              
\usepackage{graphics} 
\usepackage{graphicx}        
\usepackage{epsfig} 
\usepackage{amsmath} 
\usepackage{amssymb}  
\usepackage{hyperref}
\usepackage{graphicx}
\usepackage{subfigure}
\usepackage{color}
\usepackage{url}
\usepackage{algorithm}
\usepackage{algorithmic}

\newtheorem{remark}{Remark}
\newtheorem{proposition}{Proposition}

\newtheorem{definition}{Definition}
\newtheorem{theorem}{Theorem}
\newtheorem{problem}{Problem}
\newcommand{\ie}{{\it i.e., }}
\newcommand{\eg}{{\it e.g., }}

\newcommand{\ev}{\Diamond}
\newcommand{\gl}{\Box}
\newcommand{\un}{\mathcal U}
\newcommand{\nextltl}{\bigcirc}
\newcommand{\andltl}{\wedge}
\newcommand{\orltl}{\vee}

\newcommand{\notltl}{\neg}

\newcommand{\BA}{B\"{u}chi automaton}
\newcommand{\BAs}{B\"{u}chi automata}
\newcommand{\la}{\mathcal L}
\newcommand{\lb}{\mathcal B}
\newcommand{\lt}{\mathcal T}
\newcommand{\lc}{\mathcal C}
\newcommand{\lp}{\mathcal P}
\newcommand{\hd}{\widehat{\delta}}

\title{\LARGE \bf
Synthesis of Distributed Control and Communication Schemes\\ from Global LTL Specifications}

\author{Yushan Chen, Xu Chu Ding, and Calin Belta 
\thanks{Y. Chen is with Department of Electrical and Computer Engineering,
        Boston University, Boston, MA, 02215, U.S.A
        {\tt\small yushanc@bu.edu}}%
\thanks{X. C. Ding and C. Belta are with the Department of Mechanical Engineering, Boston University, Boston, MA, 02215, U.S.A
        {\tt\small \{xcding,cbelta\}@bu.edu}}%
\thanks{Y. Chen is the corresponding author.}%
\thanks{This work was partially supported by ONR MURI N00014-09-1051, ARO W911NF-09-1-0088, AFOSR YIP FA9550-09-1-020 and NSF CNS-0834260 at Boston University.}
}

\begin{document}
\maketitle
\thispagestyle{empty}
\pagestyle{empty}
\begin{abstract}
We introduce a technique for
synthesis of control and communication strategies for a team of agents
from a global task specification given as a Linear Temporal Logic
(LTL) formula over a set of properties that can be satisfied by the agents.
We consider a purely discrete scenario, in which the
dynamics of each agent is modeled as a finite transition system.
The proposed computational framework consists of two main steps. First,
we extend results from concurrency theory to check
whether the specification is distributable among the agents.
Second,  we generate individual control and communication strategies
by using ideas from LTL model checking.
We apply the method to automatically deploy a team of miniature cars in our
Robotic Urban-Like Environment.
\end{abstract}

\section{Introduction}
\label{sec:intro}

In control problems, ``complex" models, such as systems of differential equations, are usually checked against ``simple'' specifications, such as the stability of an equilibrium, the invariance of a set, controllability, and observability. In formal synthesis (verification), ``rich'' specifications such as languages and formulas of temporal logics are checked against ``simple'' models of software programs and digital circuits, such as (finite) transition systems. Recent studies show promising  possibilities to bridge this gap by developing theoretical frameworks and computational tools, which allow one to synthesize controllers for continuous and hybrid systems satisfying specifications in rich languages. Examples include Linear Temporal Logic (LTL) \cite{KB-TAC08-LTLCon}, fragments of LTL \cite{Hadas-ICRA07, Tok-Ufuk-Murray-CDC09}, Computation Tree Logic (CTL) \cite{Quottrup04}, mu-calculus \cite{Karaman_mu_09}, and regular expressions \cite{yushandars}.

A fundamental challenge in this area  is to construct finite models that accurately capture behaviors of dynamical systems. Recent approaches are based on the notion of abstraction \cite{alur2000discrete} and equivalence relations such as simulation and bisimulation \cite{Milner89}. Enabled by recent developments in hierarchical abstractions of dynamical systems \cite{KB-TAC08-LTLCon}, it is now possible to model systems with linear dynamics \cite{linearPappas}, polynomial dynamics \cite{TiwariHSCC}, and nonholonomic (unicycle) dynamics \cite{Lindemann06realtime} as finite transition systems.

More recent work suggests that such hierarchical abstraction techniques for a single agent can be extended to multi-agent systems, using parallel compositions \cite{Quottrup04, KB-TRO-2009}.
The two main limitations of this approach are the state space explosion problem and the need for frequent agent synchronization.
References \cite{yushandars, Hailin09} addressed both of these limitations with ``top-down'' approaches, by drawing inspirations from distributed formal synthesis \cite{mukund2002}. The main idea is to decompose a global specification into local specifications, which can then be used to synthesize controllers for the individual agents. The main drawback of these methods is that, the expressivity is limited to regular languages.

\begin{figure}
\center
   \begin{tabular}{cc}
         \includegraphics[scale = 0.35]{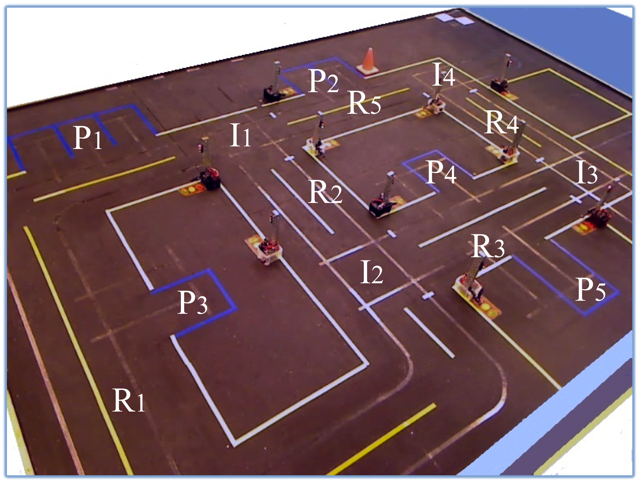}
   \end{tabular}
\caption{The topology of the Robotic Urban-Like Environment (RULE) and the road, intersection, and parking lot labels.
} \label{fig:platform}
\end{figure}

In this paper, we address a purely discrete problem, in which each agent is modeled as a finite transition system:
{\bf Given} 1) a set of properties of interest that need to be satisfied, 2) a team of agents and their capacities and cooperation requirements for satisfying properties, 3) a task specification describing how the properties need to be satisfied subject to some temporal and logical constraints in the form of an LTL formula over the set of properties; {\bf Find} provably-correct individual control and communication strategies for each agent such that the task is accomplished. Drawing inspiration from the areas of concurrency theory \cite{Mazur95} and distributed formal synthesis \cite{mukund2002},
we develop a top-down approach that allows for the fully automatic synthesis of individual control and communication schemes.  This framework is quite general and can be used in conjunction with abstraction techniques to control multiple agents with continuous dynamics.

The contribution of this work is threefold. First, we develop a computational framework to synthesize individual  control and communication strategies from global specifications given as LTL formulas over a set of interesting properties. This is a significant improvement over \cite{ yushandars} by increasing the expressivity of specifications.
Second, we extend the approach of checking closure properties
of  temporal logic specifications in \cite{Peled98} to generate distributed control and communication strategies for a team of agents while considering their dynamics. Specifically, we show how a satisfying distributed execution can be found when the global specification is traced-closed.
Third, we implement and illustrate the computational framework in our Khepera-based Robotic Urban-Like Environment (RULE) (Fig. \ref{fig:platform}). In this experimental setup, robotic cars can be automatically deployed from specifications given as LTL formulas to service requests that occur at the different locations while avoiding the unsafe regions.

The remainder of the paper is organized as follows. Some preliminaries are introduced in Sec. \ref{sec:prelim}. The problem is formulated in Sec. \ref{sec:prob_form}. An approach  for
distributing the global specification over a team of agents and
synthesizing individual control and communication strategies is presented in Sec. \ref{sec:synthesis}.
The method is applied to the RULE platform in Sec. \ref{sec:casestudy}.
We conclude with final remarks and directions for future work in Sec. \ref{sec:future}.

\section{Preliminaries}
\label{sec:prelim}
For a set $\Sigma$,  we use $|\Sigma|$, $2^{\Sigma}$, $\Sigma^*$, and $\Sigma^{\omega}$
 to denote its cardinality,  power set,  set of finite words,  and set of infinite words, respectively. We define $\Sigma^{\infty} = \Sigma^{*} \cup \Sigma^{\omega}$ and denote the empty word by $\epsilon$.
In this section, we provide background material on Linear Temporal Logic, automaton, and concurrency theory.

\begin{definition}[\textbf{transition system}]\label{definition:TS}
 A transition system (TS) is a tuple $\mathcal{T}: = (S, s_{0}, \rightarrow, \Sigma, h)$, consisting of (i) a finite set of states $S$; (ii) an initial states $s_{0} \in S$; (iii)  a transition relation $\rightarrow \subseteq S \times S$; (iv) a finite set of properties $\Sigma$; and (v) an output map $h: S \rightarrow \Sigma$.
\end{definition}

A transition $(s,s')\in\rightarrow$ is also denoted by $s \rightarrow s'$. Properties can be either true or false at each state of $\lt$.
The output map $h(s)$, where $s\in S$, defines the property valid at state $s$.
A finite {\it trajectory} of $\mathcal{T}$ is a finite sequence $r_{\lt} = s(0)s(1)\ldots s(n)$ with the property that $s(0) = s_0$ and $s(i) \rightarrow s(i + 1)$, for all $i\geq 0$.  Similarly, an infinite trajectory of $\mathcal{T}$ is an infinite sequence $r_{\lt} = s(0)s(1)\ldots$ with the same property. A finite or infinite trajectory generates a  finite or infinite {\it word} as a sequence of properties valid at each state, denoted by $w = h(s(0))h(s(1))\ldots h(s(n))$ or  $w = h(s(0))h(s(1))\ldots$, respectively.

We employ Linear Temporal Logic (LTL) formulas  to express global tasks for a team of agents. Informally,
LTL formulas are built from a set of properties $\Sigma$, standard Boolean operators $\notltl$ (negation), $\orltl$ (disjunction), $\andltl$ (conjunction), and temporal operators $\nextltl$ (next), $\un$ (until), $\ev$ (eventually), $\gl$ (always).
The semantics of LTL formulas are given over infinite words $w$ over $\Sigma$, such as those generated by a transition system defined in Def. \ref{definition:TS}.
We say an infinite trajectory $r_{\lt}$ of $\lt$ satisfies an LTL formula $\phi$ if and only if the word generated by $r_{\lt}$ satisfies $\phi$.

A word satisfies an LTL formula $\phi$ if $\phi$ is true at the first
position of the word; $\nextltl \phi$ states that at the next state, an LTL formula $\phi$ is true; $\ev \phi$ means that $\phi$ eventually becomes true
in the word; $\gl \phi$ means that $\phi$ is true at all positions of the word; $\phi_{1} \ \un \phi_{2}$ means $\phi_{2}$ eventually becomes true and $\phi_{1}$ is true until this happens. More expressivity can be achieved by combining the above temporal and Boolean operators.  Examples include $\gl\ev\phi$ ($\phi$ is true infinitely often) and $\ev\gl\phi$ ($\phi$ becomes eventually true and stays true forever).

For every LTL formula $\phi$ over $\Sigma$, there exists a \BA \ accepting all and only the words satisfying $\phi$ \cite{vardi1994reasoning}. We refer readers
to \cite{gastin2001fast} and references therein for efficient algorithms
and freely downloadable implementations to translate a LTL
formula $\phi$ to a corresponding \BA.

\begin{definition}[\textbf{\BA}]\label{def:BA}
A \BA \ is a tuple $\lb := (Q, Q^{in}, \Sigma, \delta, F)$, consisting of (i) a finite set of states $Q$; (ii) a set of initial states $Q^{in}\subseteq Q$; (iii) an input alphabet $\Sigma$; (iv) a transition function $\delta: Q \times \Sigma \rightarrow 2^{Q}$; (v) a set of accepting states $F\subseteq Q$.
\end{definition}

A {\it run} of the \BA \  over an infinite word $w= w(0)w(1)\ldots$ over $\Sigma$ is a sequence $r_{\lb} = q(0)q(1)\ldots$, such that $q(0) \in Q^{in}$ and $q(i+1) \in \delta(q(i), w(i))$. A \BA \ accepts a {\it word} $w$ if and only if  there exists $r_{\lb}$ over $w$ so that $\text{{\em inf}}(r_{\lb}) \cap F \neq \emptyset$, where $\text{{\em inf}}(r_{\lb})$ denotes the set of states appearing infinitely often in run $r_{\lb}$.
The {\it language} accepted by a \BA, denoted by  $\la(\lb)$, is the set of all infinite words accepted by $\lb$.
We use $\lb_{\phi}$ to denote the \BA\ accepting the language satisfying $\phi$.

\begin{remark}
In LTL model checking \cite{baier2008principles}, several properties can be valid at one state of a transition system (also called Kripke structure). The words produced by a transition system and accepted by a \BA \ are over the power set of propositions (\ie $2^{\Sigma}$). In this paper, by allowing only one property to be valid at a state, we consider a particular case where we allow only one property to be valid at each state of a TS by defining $h$ in Def. \ref{definition:TS} as a mapping from $S$ to $\Sigma$.
As a consequence, the words generated by $\lt$ and accepted by $\lb$ are over $\Sigma$.
\end{remark}

\begin{definition}[\textbf{distribution}]\label{def:distr}
Given a set $\Sigma$, a collection of subsets $\{\Sigma_i\subseteq\Sigma, i\in I\}$, where $I$ is an index set,
is called a distribution of $\Sigma$ if $\cup_{i \in I}\Sigma_i=\Sigma$.
\end{definition}

\begin{definition}[\textbf{projection}]\label{def:BA}
For a word $w \in \Sigma^{\infty} $ and a subset $S\subseteq \Sigma$, we denote by $w \!\upharpoonright_S$ the projection of $w$ onto $S$, which is obtained by erasing all symbols $\sigma$ in $w$ that do not belong to $\Sigma$.
For a language $L\subseteq \Sigma^{\infty}$ and a subset $S \subseteq \Sigma$, we denote by $L\!\upharpoonright_{S}$ the projection of $L$ onto $S$, which is given by $L\!\upharpoonright_{S}:= \{\omega\!\upharpoonright_S \ \mid \omega \in L\}$.
\end{definition}

\begin{definition}[\textbf{trace-closed language}]\label{definition:trace-equivalence}
Given a distribution $\{\Sigma_i\subseteq\Sigma, i\in I\}$ and $w, w' \in \Sigma^{\infty}$, we say that $w$ is trace-equivalent to $w'$ ($w \sim w'$
\footnote{Note that the trace-equivalence relation $\sim$ and class $[\cdot]$ are based on the given distribution $\{\Sigma_i\subseteq\Sigma, i\in I\}$. For simplicity of notations, we use  $\sim$ and $[\cdot]$ without specifying the distribution when there is no ambiguity.}
) if and only if $w \upharpoonright_{\Sigma_i} = w' \upharpoonright_{\Sigma_i}$, for all $i \in I$.
We denote by $[w]$ the trace-equivalence class of $w\in \Sigma^{\infty}$, which is given by $[w] := \{w' \in \Sigma^{\infty} \mid w\sim w'\}$.
A trace-closed language over a distribution $\{\Sigma_i\subseteq\Sigma, i\in I\}$ is a language $L$ such that for all $w \in L$, $[w] \subseteq L$.
\end{definition}

\begin{definition}[\textbf{product of languages}]\label{definition:prod_lang}
Given a distribution $\{\Sigma_i\subseteq\Sigma, i\in I\}$,
the product of a set of languages $L_i$ over $\Sigma_i$ is denoted by
$\parallel_{i\in I} L_i$ and defined as
$\parallel_{i\in I} L_i:=\{w \in \Sigma^{\infty} \mid w\!\upharpoonright_{\Sigma_{i}} \in L_i\textrm{ for all } i \in I\}.$
\end{definition}

\begin{proposition}\label{prop:trace-word->product}
Given a distribution $\{\Sigma_i\subseteq\Sigma, i\in I\}$ of $\Sigma$ and a
word $w \in \Sigma^{\infty}$, we have $[w] = \parallel_{i\in I} \{w\upharpoonright_{\Sigma_i}\}$.
\end{proposition}
\begin{proof}
For all words $w' \in [w]$, according to Def. \ref{definition:trace-equivalence}, $w' \upharpoonright_{\Sigma_i} = w \upharpoonright_{\Sigma_i}, \forall i \in I$. According to Def. \ref{definition:prod_lang}, since $w' \in \Sigma^{\infty}$ and $w' \upharpoonright_{\Sigma_i} = w \upharpoonright_{\Sigma_i}, \forall i \in I$, then $w' \in \parallel_{i\in I} \{w\upharpoonright_{\Sigma_i}\}$. Hence, $[w]\subseteq \parallel_{i\in I} \{w\upharpoonright_{\Sigma_i}\}$.

For all words $w'\in\parallel_{i\in I} \{w\upharpoonright_{\Sigma_i}\}$, according to Def. \ref{definition:prod_lang},
$w\upharpoonright_{\Sigma_i}=w'\upharpoonright_{\Sigma_{i}}$. According to Def. \ref{definition:trace-equivalence}, $w'\sim w$, which implies $w' \in [w]$. Hence, $\parallel_{i\in I} \{w\upharpoonright_{\Sigma_i}\} \subseteq [w]$.
Combined with the fact that $[w]\subseteq \parallel_{i\in I} \{w\upharpoonright_{\Sigma_i}\}$, we have $[w] =\parallel_{i\in I} \{w\upharpoonright_{\Sigma_i}\}$.
\end{proof}

We refer to \cite{Mazur95, Kwi89} for more definitions and properties in concurrency theory.

\section{Problem Formulation and Approach}\label{sec:prob_form}

Assume we have a team of agents $\{i\mid i \in I\}$, where $I$ is a label set.
We use an LTL formula over a set of properties $\Sigma$ to describe a global task for the team. We model the capabilities of the agents to satisfy properties as a distribution $\{\Sigma_i\subseteq\Sigma, i\in I\}$, where $\Sigma_i$ is the set of properties that  can be satisfied by agent $i$. A property can be {\it shared} or {\it individual}, depending on whether it belongs to multiple agents or to a single agent.
Shared properties are properties that need to be satisfied by several agents simultaneously.

We model each agent as a transition system:
\begin{equation}\label{eqn:robot_i}
\mathcal T_i=(S_i,s_{0_i},\rightarrow_i, \Sigma_i, h_i), i\in I.
\end{equation}
In other words, the dynamics of agent $i$ are restricted by the transition relation $\rightarrow_{i}$.
The output $h_{i}(s_{i})$ represents the property  that is valid (true) at state $s_{i} \in S_{i}$. An individual property $\sigma$ is said to be satisfied if and only if the agent that owns $\sigma$ reaches state $s_{i}$ at which $\sigma$ is valid (\ie, $h_i(s_{i}) = \sigma$). A shared property is said to be satisfied if and only if all the agents sharing it enter the states where $\sigma$ is true simultaneously.

For example, $\lt_{i}$ can be used to model the motion capabilities of a robot (Khepera III miniature car) running in our urban-like environment (Fig. \ref{fig:platform}), where $S_{i}$ is a set of labels for the roads, intersections and parking lots and $\rightarrow_{i}$ shows how these are connected (\ie  $\rightarrow_{i}$ captures how robot $i$ can move among adjacent regions).
Note that these transitions are, in reality, enabled by low-level control primitives (see Sec. \ref{sec:casestudy}).
We assume that the selection of a control primitive at a region uniquely determines the next region. This corresponds to a deterministic (control) transition system, in which each trajectory of  $\lt_{i}$ can be implemented by the robot in the environment by using the sequence of corresponding motion primitives. For simplicity of notation, since the robot can deterministically choose a transition, we omit the control inputs traditionally associated with transitions. Furthermore, distribution $\{\Sigma_i\subseteq\Sigma, i\in I\}$ can be used to capture the capabilities of the robots to service requests and task cooperation requirements (\eg some of the requests can be serviced by one robot, while others require the collaboration of two or more robots).  The output map $h_{i}$ indicates the locations of the requests. A robot services a request by visiting the region at which this request occurs. A shared request occurring at a given location requires multiple robots to be at this location at the same time.

\begin{definition}[\textbf{cc-strategy}]
A finite (infinite) trajectory $r^c_i =  s_{i}(0)s_{i}(1)\ldots s_{i}(n)$ ($s_{i}(0)s_{i}(1)\ldots$) of $\lt_{i}$ defines a control and communication (cc) strategy for agent $i$ in the following sense:
(i) $s_{i}(0) = s_{0_i}$, (ii) an entry $s_{i}(k)$ means that state $s_{i}(k)$ should be visited, (iii) an entry $s_{i}(k)$, where $h_i(s_{i}(k))$ is a shared property, triggers a communication protocol:
while at state $s_{i}(k)$, agent $i$ broadcasts the property $h_i(s_{i}(k))$ and listens for broadcasts of $h_i(s_{i}(k))$ from all other agents that share the property with it; when they are all received, $h_i(s_{i}(k))$ is satisfied and then agent $i$ transits to the next state.
\end{definition}

Because of the possible parallel satisfaction of individual properties, and because the durations of the transitions are not known,
 a set of cc-strategies $\{r^{c}_i,i \in I\}$ can produce multiple sequences of properties satisfied by the team. We use products of languages (Def. \ref{definition:prod_lang}) to capture all the possible behaviors of the team.

\begin{definition}[\textbf{global behavior of the team}]
Given a set of cc-strategies $\{r^{c}_i, i\in I\}$, we denote
\begin{equation}\label{eqn:team}
\la_{team} (\{r^{c}_i, i\in I\}) :=\  \parallel_{i\in I} \{w_{i}\}
\end{equation}
as the set of all possible sequences of properties satisfied by the team
while the agents follow their individual cc-strategies $r^{c}_i$, where $w_{i}$ is the word of $\lt_i$ generated by
$r^{c}_i$.
\end{definition}

For simplicity of notation, we usually denote $\la_{team} (\{r^{c}_i, \\i\in I\})$ as $\la_{team}$ when there is no ambiguity.

 \begin{definition}[\textbf{satisfying set of cc-strategies}]
 A set of cc-strategies $\{r^{c}_i, i\in I\}$ satisfies a specification given as an LTL formula $\phi$ if and only if $\la_{team}\neq \emptyset$ and $\la_{team} \subseteq \mathcal L(\lb_{\phi})$.
\end{definition}

\begin{remark}\label{remark:deadlock}
For a set of cc-strategies, the corresponding $\la_{team}$ could be an empty set by the definition of product of languages (since there may not exist a word $w\in \Sigma^{\infty}$ such that $w\upharpoonright_{\Sigma_i}=w_{i}$ for all $i\in I$). In practice, this case corresponds to a deadlock scenario where one (or more) agent waits indefinitely for others to enter the states at which a shared property $\sigma$ is true. For example, if one of these agents is not going to broadcast $\sigma$ but some other agents are waiting for the broadcasts of $\sigma$, then all those agents will be stuck in a ÒdeadlockÓ state and wait indefinitely. When such a deadlock scenario occurs, the behaviors of the team do not satisfy the specification.
\end{remark}

We are now ready to formulate the main problem:

\begin{problem}\label{problem:main}
Given a team of agents represented by $\lt_i$, $i\in I$, a global specification $\phi$ in the form of an LTL formula over $\Sigma$, and
a distribution $\{\Sigma_i\subseteq\Sigma, i\in I\}$, find a satisfying set of individual cc-strategies $\{r^{c}_{i}, i\in I\}$.
\end{problem}

Our approach to solve Prob. \ref{problem:main} can be divided into two major parts as shown in Fig \ref{fig:approach}: checking distributability and ensuring implementability.
Specifically, we (i) check whether the global specification can be distributed among the agents while accounting for their capabilities to satisfy properties, and (ii) make sure that the individual cc-strategies are feasible for the agents. For (i), we make the connection between distributability of global specifications and closure properties of temporal logic formulas \cite{Peled98}.  Specifically, we check whether the language satisfying the global specification $\phi$ is trace-closed; if yes, then it is distributable; otherwise, a solution cannot be found (see Sec. \ref{sec:sub:dist}). Therefore, our approach is conservative, in the sense that we might not find a solution even if one exists.  For (ii), we construct an implementable automaton by adapting automata-based techniques \cite{Clarke99, thiaTech} to obtain all the possible sequences of properties that could be satisfied by the team, while considering the dynamics and capabilities of the agents (Sec. \ref{sec:sub:imple} and \ref{sec:sub:product}).  Finally, an arbitrary word from the intersection of the trace-closed language satisfying $\phi$ and the language of the implementable automaton is selected to synthesize the individual cc-strategies for the agents.

\begin{figure}\center
\includegraphics[scale=0.5]{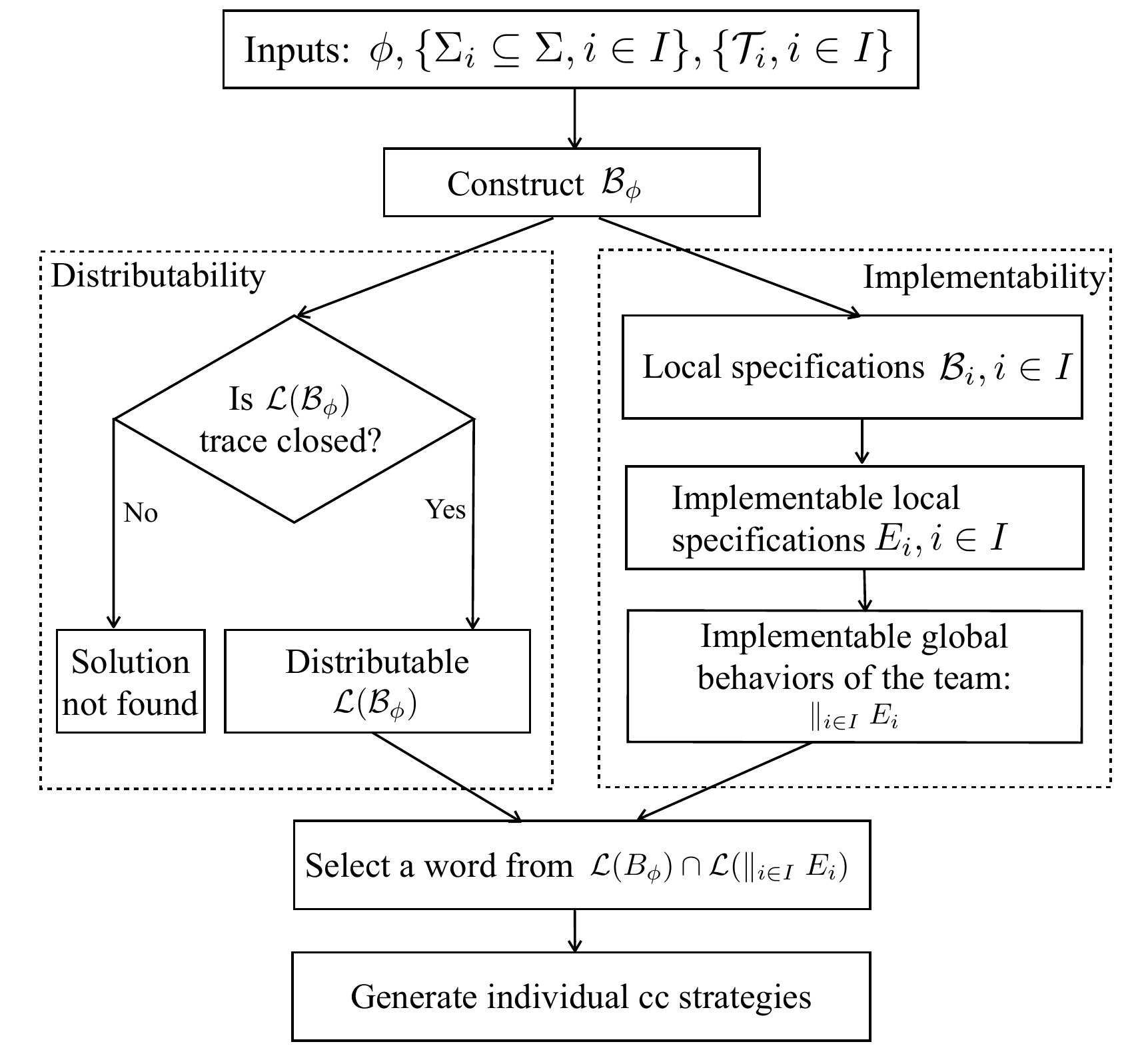}
 \caption{Schematic representation of our approach to Prob. \ref{problem:main}. } \label{fig:approach}
\end{figure}

\section{Synthesis of individual cc-strategies}\label{sec:synthesis}

\subsection{Checking Distributability}\label{sec:sub:dist}

We begin with the conversion of the global specification $\phi$ over $\Sigma$ to a \BA \
$\lb_{\phi} = (Q, Q^{in}, \Sigma, \delta, F)$ (Def. \ref{def:BA}), which accepts exactly the language satisfying $\phi$ (using LTL2BA \cite{gastin2001fast}).
We need to find a local word $w_{i}$ for each agent $i$ such that (i) all possible sequences of properties satisfied by the team while each agent executes its local word satisfy the global specification (\ie included in $\la(\lb_{\phi})$), and (ii) each local word $w_{i}$ can be implemented by the corresponding agent (which will be detailed in the following sub-sections).

 Given the global specification $\la(\lb_{\phi})$ and the distribution $\{\Sigma_i\subseteq\Sigma, i\in I\}$,  we make the important observation that a trace-closed language (Def. \ref{definition:trace-equivalence}) is sufficient to find a set of local words satisfying the first condition.  Formally, we have:

\begin{proposition}\label{prop:trace->solution}
Given a language $L \subset \Sigma^{\infty}$ and a distribution $\{\Sigma_i\subseteq\Sigma, i\in I\}$, if $L$ is a trace-closed language and $w \in L$, then $\parallel_{i} \{w\upharpoonright_{\Sigma_i}\} \subseteq L$.
\end{proposition}

\begin{proof}
Follows from Prop. \ref{prop:trace-word->product} and the definition of the trace-closed language.
\end{proof}

Thus, our approach aims to check whether $\la(\lb_{\phi})$ is trace-closed. If the answer is positive, by Prop. \ref{prop:trace->solution}, an arbitrary word from $\la(\lb_{\phi})$ can be used to generate the suitable set of local words by projecting this word onto $\Sigma_i$. The algorithm (adapted from \cite{Peled98}) to check if $\mathcal L(\lb_{\phi})$ is trace-closed can be viewed as a process to construct a \BA \  $\mathcal A$, such that each word accepted by  $\mathcal A$ represents a pair of words $w$ and $w'$, such that $w \in \la (\lb_{\phi})$, $w' \notin \la (\lb_{\phi})$, and $w \sim w'$ (\ie $w$ is trace-equivalent to $w'$).  Thus, if $\mathcal A$ has a non-empty language, $\la (\lb_{\phi})$ is not trace-closed.

To obtain $\mathcal A$, we first construct a \BA, denoted by $\mathcal C$, to capture all pairs of trace-equivalent infinite words over $\Sigma$. Given the distribution $\{\Sigma_i\subseteq\Sigma, i\in I\}$, we define a relation  $\mathbb I$ such that $(\sigma, \sigma') \in\mathbb I$ if there does not exist $\Sigma_{i}$, $i\in I$ such that  $\sigma, \sigma' \in \Sigma_{i}$.  Formally, $\mathcal C$ is defined as
\begin{equation}\label{eqn:automatonC}
\lc = (Q_{\lc}, \{q_{\lc_{0}}\}, \Sigma_{\lc}, \delta_{\lc}, F_{\lc}),
\end{equation}
where $\Sigma_{\lc} = \mathbb I \cup \{(\sigma, \sigma)\mid \sigma \in \Sigma\}$ and $F_{\lc} = \{q_{\lc_{0}}\}$.
The transition function $\delta_{\lc}$ is defined as
(a)  for all $\sigma \in \Sigma$, there exists $q_{\lc_{0}} = \delta_{\lc}( q_{\lc_{0}}, (\sigma, \sigma))$,
and (b) for all $(\sigma, \sigma') \in \mathbb I$, there exists a state $q_{\lc} \neq q_{\lc_{0}}$ such that $q_{\lc} = \delta_{\lc} (q_{\lc_{0}}, (\sigma, \sigma')) $ and $q_{\lc_{0}} =  \delta_{\lc} (q_{\lc}, (\sigma', \sigma))$.
 In other words, to obtain $\lc$, we first generate the initial state and then add a new state and the corresponding transitions for every member of $\mathbb I$. Thus, the number of states is $|\mathbb I | + 1$. A simple example to illustrate the construction of $\lc$ is shown in Fig. \ref{fig:C}.

\begin{figure}\center
\includegraphics[scale=3.4]{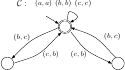}
\caption{
\BA \ $\lc$ (\ref{eqn:automatonC}) for the case when $\Sigma = \{a,b,c\}$, $\Sigma_1= \{a,b\}$, and $\Sigma_2= \{a,c\}$. Relation $\mathbb I$ is given by $\mathbb I = \{(b,c), (c,b)\}$.
}\label{fig:C}
\end{figure}

Next, we construct a \BA \ $\mathcal A_{1}$ to accommodate words from $\mathcal L(\lb_{\phi})$.  A word $w_{\mathcal A_{1}}$ accepted by $\mathcal A_{1}$ is a sequence $(\sigma_{1},\sigma'_{1})(\sigma_{2}, \sigma'_{2})\ldots$. We use $w_{\mathcal A_{1}}|_{1}$ and $w_{\mathcal A_{1}}|_{2}$ to denote the sequence $\sigma_{1}\sigma_{2}\ldots$ and $\sigma'_{1}\sigma'_{2}\ldots$, respectively. For each word $w_{\mathcal A_{1}}$ accepted by $\mathcal A_{1}$, we have $w_{\mathcal A_{1}}|_{1} \in \mathcal L(\lb_{\phi})$ and $w_{\mathcal A_{1}}|_{2} \in \Sigma^{\omega}$. Similarly, we construct another \BA \ $\mathcal A_{2}$ to capture words that do not belong to $\mathcal L(\lb_{\phi})$, \ie  for each word $w_{\mathcal A_{2}} \in \la(\mathcal A_{2})$, $w_{\mathcal A_{2}}|_{1} \in \Sigma^{\omega}$ and $w_{\mathcal A_{2}}|_{2} \notin \mathcal L(\lb_{\phi})$ always hold.

Finally, we produce the \BA \ $\mathcal A$ such that $\mathcal L (\mathcal A) =  \mathcal L (\mathcal C)  \cap \mathcal L (\mathcal A_{1})  \cap \mathcal L (\mathcal A_{2}) $ by taking the intersections of the \BAs. According to \cite{Peled98}, $\mathcal L(\lb_{\phi})$ is trace-closed if and only if $\la(\mathcal A) = \emptyset$.
The construction of the intersection of several  \BAs \ is given in \cite{vardi1994reasoning}.
We summarize this procedure in Alg. \ref{alg:trace}.

\begin{algorithm}
\caption{: Check if $\la(\lb)$ is trace-closed}
\label{alg:trace}
\begin{algorithmic}[1]
\INPUT{A \BA \ $\lb = (Q, Q^{in}, \Sigma, \delta, F)$ (Def. \ref{def:BA})} and a distribution $\{\Sigma_i\subseteq\Sigma, i\in I\}$
\OUTPUT{Yes or No}
\STATE Construct $\lc$ as defined in (\ref{eqn:automatonC})
\STATE Construct $\mathcal A_{1} = (Q, Q^{in}, \Sigma_{\mathcal A_1}, \delta_{\mathcal A_1}, F)$, where $\Sigma_{\mathcal A_1} \subseteq \Sigma \times\Sigma$ and $ \delta_{\mathcal A_1}: Q \times \Sigma_{\mathcal A_1} \rightarrow 2^{Q}$ is defined as $q' \in  \delta_{\mathcal A_1}(q, (\sigma_{1}, \sigma_{2})) $ if and only if $q' \in  \delta(q, \sigma_{1})$
\STATE Construct $\mathcal A_{2} = (Q, Q^{in}, \Sigma_{\mathcal A_2}, \delta_{\mathcal A_2}, F)$, where $\Sigma_{\mathcal A_2} \subseteq \Sigma \times\Sigma$ and $ \delta_{\mathcal A_2}: Q \times \Sigma_{\mathcal A_2}  \rightarrow 2^{Q}$ is defined as $q' \in  \delta_{\mathcal A_2}(q, (\sigma_{1}, \sigma_{2})) $ if and only if $q' \in  \delta(q, \sigma_{2})$.
\STATE Construct $\mathcal A$ such that  $\mathcal L (\mathcal A) =  \mathcal L (\mathcal C)  \cap \mathcal L (\mathcal A_{1})  \cap \mathcal L (\mathcal A_{2})$
\STATE \textbf{if} $\la(\mathcal A) = \emptyset$ \textbf{return} Yes \textbf{else} \textbf{return} No
\end{algorithmic}
\end{algorithm}

\subsection{Implementable Local  Specification}\label{sec:sub:imple}

In the case that $\la(\lb_{\phi})$ is trace-closed, the global specification is distributable among the agents. We call $\la(\lb_{\phi})\upharpoonright_{\Sigma_i}$
the ``local'' specification for agent $i$ because of the following proposition.

\begin{proposition}
If a set of cc-strategies $\{r^{c}_{i}, i\in I\}$ is a solution to Prob. \ref{problem:main}, then the corresponding local words $w^{c}_{i}$ are included in $ \mathcal L(\lb_{\phi})\upharpoonright_{\Sigma_i}$ for all $i\in I$.
\end{proposition}

\begin{proof}
If a set of cc-strategies $\{r^{c}_{i}, i\in I\}$ is a solution to Prob. \ref{problem:main}, then we have $||_{i\in I} \{w^{c}_{i}\} \subseteq \mathcal L(\lb_{\phi})$ and $||_{i\in I} \{w^{c}_{i}\} \neq \emptyset$.  We can find a word $w_{1} \in ||_{i\in I} \{w^{c}_{i}\} \subseteq \mathcal L(\lb_{\phi})$, such that $w^{c}_{i}=w_{1}\upharpoonright_{\Sigma_i}$ for all $i\in I$.  Since $w^{c}_{i}=w_{1}\upharpoonright_{\Sigma_i}$ and $w_{1}\upharpoonright_{\Sigma_i}\in \mathcal L(\lb_{\phi})\upharpoonright_{\Sigma_i}$, we have $w^{c}_{i}\in \mathcal L(\lb_{\phi})\upharpoonright_{\Sigma_i}$.
\end{proof}

Given the agent model $\lt_i$, some of the local words might not be feasible for the agent. Therefore, we aim to construct the ``implementable local'' specification for each agent; namely, it captures all the words of $ \mathcal L(\lb_{\phi})\upharpoonright_{\Sigma_i}$ that can be implemented by the agent.
To achieve this, we first produce an automaton that accepts exactly the local specification.

Note that the projection of the language satisfying the global specification that includes only infinite words on a local alphabet $\Sigma_{i}$ might contain finite words. For example, given an infinite word $w = baaa\ldots$, if $a \notin \Sigma_{i}$, the projection of this word is $b$. Therefore, the local specification for each agent might have both finite and infinite words.  To address this, we employ a {\it mixed \BA} introduced in \cite{thiaTech}. The mixed \BA \ is similar to the standard \BA \ defined in Def. \ref{def:BA}, except for it has two different types of accepting states: finitary and infinitary accepting states. Formally, we define the mixed \BA \ as
 \begin{equation}\label{def:mix}
 \lb^{M} := (Q, Q^{in}, \Sigma, \delta, F, F^{fin})
\end{equation}
where $F$ stands for the set of infinitary accepting states and $F^{fin}$ represents the set of finitary  accepting states.  The mixed \BA \ accepts infinite words by using the set of infinitary accepting states, with the same acceptance condition as defined in Def. \ref{def:BA}. A finite run $r^{fin} = q(0)q(1)\ldots q(n)$ of $\lb^{M}$ over a finite word $w^{fin} = w(0)w(1)\ldots w(n)$ satisfies $q(0) \in Q^{in}$ and $q(i+1) \in \delta(q(i), w(i))$, for all $0\leq i < n$. $\lb^{M}$ accepts a finite word $w^{fin}$ if and only if the finite run $r^{fin}$ over $w^{fin}$ satisfying $ q(n) \in F^{fin}$. We call a finitary  accepting state $q\in F^{fin}$ terminal if and only if no transition starts from $q$. We assume that all the  finitary accepting states are terminal in this paper.
An algorithm to obtain a mixed \BA \ $\lb_i = (Q_i, Q^{in}_i, \Sigma_i, \delta_i, F_i, F^{fin}_i)$ which accepts $\la(\lb_{\phi})\upharpoonright_{\Sigma_i}$ is summarized in Alg. \ref{alg:projection}.

\begin{algorithm}
\caption{: Construct $\lb_i$ where $\la(\lb_i) = \la(\lb) \upharpoonright_{\Sigma_{i}}$}
\label{alg:projection}
\begin{algorithmic}[1]
\INPUT{$\lb = (Q, Q^{in}, \Sigma, \delta, F)$ and a subset $\Sigma_i \subseteq \Sigma$}
\OUTPUT{$\lb_i =  (Q_{i},Q^{in}_{i},\Sigma_{\lb_{i}},\delta_{i}, F_{i}, F^{fin}_i)$ }

\STATE  Construct  $\lb^{\epsilon}_{i} = (Q^{\epsilon}_{i},Q^{{\epsilon}_{in}}_{i},\Sigma^{\epsilon}_i,\delta^{\epsilon}_{i}, F^{\epsilon}_{i})$, where $Q^{\epsilon}_{i} = Q$, $Q^{{\epsilon}_{in}}_{i} = Q^{in}$, $\Sigma^{\epsilon}_i = \Sigma_i \cup \{\epsilon\}$,  $F^{\epsilon}_{i} = F$ and
$\delta^{\epsilon}_{i}$ is defined as $ q' \in \delta^{\epsilon}_{i}(q, \sigma)$ iff $q' \in \delta(q, \sigma)$ and $\sigma \in \Sigma_i$, and $q' \in \delta^{\epsilon}_{i}(q, \epsilon)$ iff $\exists \sigma \in \Sigma\backslash\Sigma_i$ {\it s.t.,} $q' \in \delta(q, \sigma)$.
\STATE For all states $q$ of  $\lb^{\epsilon}_{i}$, we take the $\epsilon$-closure \cite{automata-book07} of $q$, denoted as $eclose(q)$.
\STATE Build $\lb_{i} = (Q_{i},Q^{in}_{i},\Sigma_{\lb_{i}},\delta_{i}, F_{i}, F^{fin}_i)$, where $Q_{i} = Q^{\epsilon}_{i}$, $Q^{in}_{i} = Q^{{\epsilon}_{in}}_{i}$, $\Sigma_{\lb_{i}} = \Sigma_i$,   $\delta_{i}$ is defined as $q' \in \delta_{i}(q, \sigma)$, iff $\exists q'' \in eclose(q),$  {\it s.t.,} $q' \in \delta^{\epsilon}_{i}(q'', \sigma)$, $F_{i} = F^{\epsilon}_{i}$ and  $F^{fin}_{i} = \emptyset$.
\STATE Obtain $F^{fin}_i$ by adding a new state $q^{fin}$ to $F^{fin}_i$
for each $q \in F_{i}$ where a loop $q \xrightarrow{\epsilon} q_{1}  \xrightarrow{\epsilon} q_{2} \ldots\xrightarrow{\epsilon} q$ in $\lb^{\epsilon}_{i}$  exists
\STATE Add $F^{fin}_i$  to $Q_i$.
\STATE For each state $q^{fin} \in F^{fin}_i$, we have  $q^{fin} \in \delta_{i}(q', \sigma)$ if and only if the state's corresponding state of $q \in F_i$ satisfying $q \in \delta_{i}(q', \sigma)$
\RETURN $\lb_{i} = (Q_{i},Q^{in}_{i},\Sigma_{\lb_{i}},\delta_{i}, F_{i}, F^{fin}_i)$
\end{algorithmic}
\end{algorithm}

\begin{proposition}\label{prop:projection}
The language of the mixed \BA \ $\lb_i = (Q_i, Q^{in}_i, \Sigma_i, \delta_i, F_i, F^{fin}_i)$ constructed in Alg. \ref{alg:projection} is equal to $\la(\lb_{\phi}) \upharpoonright_{\Sigma_i}$.
\end{proposition}
\begin{proof}
By construction, $\lb^{\epsilon}_{i}$ accepts $\la(\lb) \upharpoonright_{\Sigma_i}$. To prove the above proposition, we first prove the following statement:  $\lb_{i}$ obtained by Alg. \ref{alg:projection} accepts the same infinite language as $\lb^{\epsilon}_{i}$ does.
For the infinite language, we only need to consider $\lb_i$ constructed in step 3 of the algorithm since step 4, 5, and 6 are only related to the finite language. From now on, $\lb_i = (Q_i, Q^{in}_i, \Sigma_i, \delta_i, F_i, F^{fin}_i)$ refers to $\lb_i$ constructed in step 3.

We define $\hd_{i}(Q_{1}, w)$, $Q_{1} \subseteq Q_{i}$ inductively to represent a set of states that can be reached from $Q_{1}$ after taking $w= w(1)w(2)\ldots w(n)$ as inputs.
Formally, we define $\hd_{i}$ for a \BA's transition function $\delta_i$ by:

\textbf{Basis:} $\hd_{i}(Q_{1}, \epsilon) = Q_{1}$. That is, without reading any input symbols, we are only in the state we began in.

\textbf{Induction:} Suppose $w$ is of the form $w = xa$, where $a$ is the final symbol of $w$ and $x$ is the rest of $w$. Also suppose that $\hd_{i}(Q_{1}, x) =
\{q_{1}, q_{2}, \ldots, q_{k}\}$. Let
$$
\bigcup^{k}_{j = 1} \delta_{i}(\{q_{j}\}, a) = \{r_{1}, r_{2}, \ldots, r_{m}\}
$$
Then  $\hd_{i}(Q_{1}, w) =  \{r_{1}, r_{2}, \ldots, r_{m}\}$. Less formally, we compute  $\hd_{i}(Q_{1}, w)$ by first computing $\hd_{i}(Q_{1}, x)$, and then following any transition from any of these states that is labeled $a$.

Similarly, for the \BA \ with $\epsilon$-transitions, $\hd_{i}^{\epsilon}(Q_{2}, w)$, $Q_{2} \subseteq Q^{\epsilon}_{i}$, is defined to represent the set of states, which can be reached from the set of the states $Q_{2}$ after taking a sequence of transitions given the input sequence $w$, while accounting for the transitions that can be made spontaneously (\ie $\epsilon$-transitions).
With slight abuse of notation, we denote $\delta^{\epsilon}_{i}(Q_{2}, a) = \bigcup_{q\in Q_{2}} \delta^{\epsilon}_{i}(q,a)$.
Formally, we define $\hd_{i}^{\epsilon}$ for the transition function $\rightarrow^{\epsilon}_{\lb_{i}}$ of  a \BA \ with $\epsilon$-transitions as following:

\textbf{Basis:} $\hd_{i}^{\epsilon}(Q_{2}, \epsilon) = Q_{2}$.

\textbf{Induction:} Suppose $w$ is of the form $w = xa$. Also suppose that $\hd_{i}^{\epsilon}(Q_{2}, x) = \{q_{1}, q_{2}, \ldots, q_{k}\}$. Let
\begin{align*}
\bigcup^{k}_{j = 1} \hd_{i}^{\epsilon}(\{q_{j}\}, a) &= \bigcup^{k}_{j = 1} \delta_{i}^{\epsilon}(eclose(q_{j}), a) )\\
& = \{r_{1}, r_{2}, \ldots, r_{m}\}
\end{align*}

Then  $\hd_{i}^{\epsilon}(Q_{2}, w) =  \{r_{1}, r_{2}, \ldots, r_{m}\}$.
Less formally, we compute  $\hd_{i}^{\epsilon}(Q_{2}, w)$ by first computing $\hd_{i}^{\epsilon}(Q_{2}, x)$, then following any $\epsilon$-transition from any of these states, and finally following any transition from the reached states that is labeled $a$.

To prove the statement, what we prove first, by induction on $|w|$, where $w = w(1)w(2)\ldots w(n) \in \Sigma^*$, is that
\begin{equation}\label{eqn:reach}
\hd_{i}(Q^{in}_{i}, w) =  \hd_{i}^{\epsilon}(Q^{{\epsilon}_{in}}_{i}, w).
\end{equation}

\textbf{Basis:} Let $|w| = 0$; that is, $w = \epsilon$. By the basis definitions of $\hd_{i}$ and  $\hd_{i}^{\epsilon}$, $\hd_{i}(Q^{in}_{i}, w) = Q^{in}_{i}$ and $\hd_{i}^{\epsilon}(Q^{{\epsilon}_{in}}_{i}, w) = Q^{\epsilon_{in}}_{i}$. Since $Q^{in}_{i} = Q^{\epsilon_{in}}_{i} = Q_{i}$, (\ref{eqn:reach}) holds.

\textbf{Induction:} Let $w$ be of length $n+1$, and assume (\ref{eqn:reach}) for length $n$. Break $w$ as $w = xa$. Let the set of states in $Q_{i}$ be $\{q_{1}, q_{2}, \ldots, q_{k}\}$ and the set of states in $Q^{\epsilon}_{i}$ be $\{q^{\epsilon}_{1}, q^{\epsilon}_{2}, \ldots, q^{\epsilon}_{k}\}$, and $q^{\epsilon}_{j} = q_{j}$,  $1 \leq j \leq k$.

By the construction of $\rightarrow_{\lb_{i}}$, we have $q \in \delta_{i}(\{q_{j}\}, a) $ if and only if $q \in \delta^{\epsilon}_{i}(eclose(q_{j}), a)$. By definition, since
$\delta_{i}(\{q_{j}\}, a) =  \delta_{i}^{\epsilon}(eclose(q_{j}), a) )$ and $q^{\epsilon}_{j} = q_{j}$,  we have $\bigcup^{k}_{j = 1} \delta_{i}(\{q_{j}\}, a) = \bigcup^{k}_{j = 1} \hd_{i}^{\epsilon}(\{q^{\epsilon}_{j}\}, a)$. Therefore, we have $\hd_{i}(Q^{in}_{i}, w) =  \hd_{i}^{\epsilon}(Q^{{\epsilon}_{in}}_{i}, w)$.

When we observe that $\lb_{i}$ constructed in step 3 and $\lb^{\epsilon}_{i}$ accept an infinite word if and only if this word visits the accepting states $F_{i}$ and $F^{\epsilon}_{i}$ infinitely many time. Since $F_{i} = F^{\epsilon}_{i}$, and $\hd_{i}(Q^{in}_{i}, w) =  \hd_{i}^{\epsilon}(Q^{{\epsilon}_{in}}_{i}, w)$, we have a proof that the two \BAs \ accept the same infinite language.

Next, we consider the finite language. From now on, $\lb_i = (Q_i, Q^{in}_i, \Sigma_i, \delta_i, F_i, F^{fin}_i)$ refers to $\lb_i$ returned by the algorithm.  Note that a finite word is accepted by $\lb^{\epsilon}_{i}$  if and only if its corresponding run ends at one of the accepting states $q^{\epsilon} \in F^{\epsilon}_{i}$, such that there exists a loop starting from and ending at it, with only $\epsilon$-transitions. By the construction of $\lb_{i}$, for the state $q^{\epsilon}$, there exist two corresponding  states: $q \in F_{i}$ and $q^{fin} \in F^{fin}_{i}$. Note that a run over a word can reach $q^{fin}$ if and only if it can reach $q$.
Because of (\ref{eqn:reach}), a finite word, whose corresponding run on $\lb_{i}$ can reach $q \in  F_{i}$ and $q^{fin} \in F^{fin}_{i}$ if and only if its corresponding run on $\lb^{\epsilon}_{i}$ can reach $q^{\epsilon} \in F^{\epsilon}_{i}$, which implies that this finite word is accepted by both \BAs.
Hence, we have a proof that the two \BAs \ accept the same finite language. Since $\lb^{\epsilon}_{i}$ and $\lb_{i}$ have the same language, the proof is complete.
\end{proof}

Inspired from LTL model checking \cite{Clarke99}, we define a product automaton to obtain the implementable local specification. First, we extend the transition system $\lt_i$ with a dummy
state labelled as $Start$ that has a transition to the initial state $s_{0_i}$. The addition of this
dummy state is necessary in the case that the initial state already satisfies partially the local specification.
Let $\widehat{\lt_i}$ be the extended finite transition system, then 
\begin{equation}\label{eqn:hatT}
\widehat{\lt_i}=(\widehat{S_i},\widehat{s_{0_i}},\widehat{\rightarrow_i}, \widehat{\Sigma}, \widehat{h_i})
\end{equation}
 where $\widehat{S_i} = S_i \cup \{Start\}$, $\widehat{s_{0_i}} = Start$, 
$\widehat{\rightarrow_i}= \rightarrow_i \cup \rightarrow_s$ where $\rightarrow_s$ is defined as $Start \rightarrow_s s_{0_i}$, $\widehat{\Sigma} = \Sigma$ and $\widehat{h_i} $ is the same output map as $h_i$ but extended by mapping the $Start$ state to a dummy observation.
Note that $\widehat{\lt_i}$ and ${\lt_i}$ generate the same language.

Now, consider the transition system $\widehat{\lt_i}$ that describes the
dynamics of agent $i$ and $\lb_i$ that represents the local specification for agent $i$. 
The following product automaton captures all the words in $\la(\lb_{i})$ that can be generated by agent $i$.

\begin{definition}\label{definition:PA}
The product automaton $E_i = \widehat{\lt_{i}} \otimes \lb_{i}$ between a TS $\widehat{\lt_i}=(\widehat{S_i},\widehat{s_{0_i}},\widehat{\rightarrow_i}, \widehat{\Sigma}, \widehat{h_i})$
and a mixed \BA \ $\lb_{i} = (Q_{i},Q^{in}_{i},\Sigma_{\lb_{i}},\delta_{i}, F_{i}, F^{fin}_{i})$, is a mixed  \BA \  $E_i = (Q_{E_i},Q^{in}_{E_i},\Sigma_{E_{i}}, \delta_{E_i},F_{E_i}, F^{fin}_{E_i})$, consisting of
\begin{itemize}
\item  a set of states $Q_{E_i} = \widehat{S_i} \times Q_{i}$;
\item a set of initial states $Q^{in}_{E_i} = \widehat{s_{0_i}} \times Q^{in}_{i}$;
\item a set of inputs $\Sigma_{E_i} = \Sigma_{\lb_i} $;
\item a transition function $\delta_{E_i}$ defined as $(s',q') \in \delta_{E_i}((s,q), h_{i}(s'))$ iff $s\widehat{\rightarrow_i} s'$ and $q' \in \delta_{i}(q, h_{i}(s'))$;
\item a set of infinitary accepting states $F_{E_i} = \widehat{S_i} \times F_i$;
\item a set of finitary accepting states $F^{fin}_{E_i} = \widehat{S_i} \times F^{fin}_i$.
\end{itemize}
\end{definition}

 \begin{figure*}\center
\includegraphics[scale=0.5]{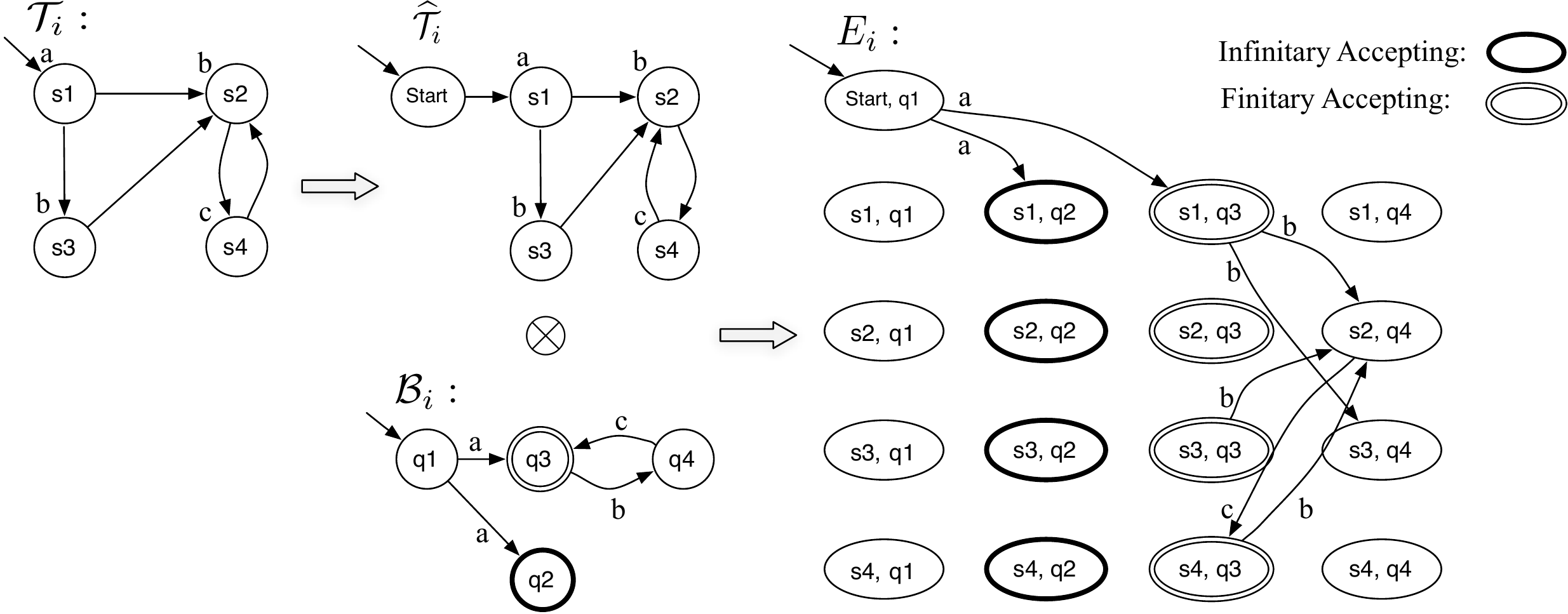}
\caption{An example of constructing $E_{i}$ given a \BA \ $\lb_{i}$ and a transition system $\lt_{i}$. As we can see in $E_i$, some states are unreachable from the initial state or cannot reach any accepting state. Such states can be removed to reduce the size of $E_{i}$.
}\label{fig:prod}
\end{figure*}

Informally, the \BA \ $\lb_i$ restricts the behavior of the transition system $\widehat{\lt_i}$ by permitting only certain acceptable transitions.
Note that we modify the traditional definition of product automata  \cite{baier2008principles} to accommodate the finitary accepting states.  
An example showing how to construct the product automaton given a transition system and a mixed \BA\  is illustrated in Fig. \ref{fig:prod}.
The following proposition shows that $\la(E_{i})$ is exactly the implementable local specification for agent $i$.

\begin{proposition}\label{prop:runofP}
Given any accepted word $w$ of $\lb_{i}$, there exist at least one trajectory of $\lt_{i}$ generating $w$ if and only if $w\in \la(E_{i})$.
\end{proposition}

\begin{proof}
``$\Longleftarrow$'':
Given an infinite word $w \in  \mathcal{L} (E_i)$, there exists an infinite run $r_{E_i}=(Start, q_{i}(1))(s_{i}(1), q_{i}(2))\ldots$ of $E_i$ which generates $w$, where $s(1) = s_{0_{i}}$.
We define the projection of $r_{E_i}$ onto $\lt_i$ as $\gamma_{\lt_i}(r_{E_i}) =s_{i}(1)s_{i}(2)\ldots $.  By the definition of the product automaton, $\gamma_{\lt_i}(r_{E_i})$ is an infinite trajectory of $\lt_{i}$ generating $w$, which is a word of $\lb_{i}$.

Given a finite word $w^{fin} \in \la(E_{i})$ with length $k$, there exists a finite run in the form of $r_{E_i}=(Start, q_{i}(1))(s_{i}(1), q_{i}(2))\ldots(s_{i}(k), q_{i}(k+1))$ of $E_i$ which generates $w$.  The projection of $r_{E_i}$ on $\lt_{i}$ can be written as $ s_{i}(1)s_{i}(2)\ldots s_{i}(n)$.  By the definition of the product automaton, $ s_{i}(1)s_{i}(2)\ldots s_{i}(n)$ is a finite trajectory of $\lt_{i}$ generating the finite word $w^{fin}$, which means there exists a trajectory of $\lt_{i}$ generating $w^{fin} \in \la(\lb_{i})$.

``$\Longrightarrow$'':
Given an infinite word $w= w(1)w(2)\ldots$ accepted by $\lb_i$ and a trajectory $r_{\lt_i}=s_{i}(1)s_{i}(2)\ldots$ of $\lt_i$ satisfying $w$, then we have $s_{i}(j) \rightarrow_i s_{i}(j+1)$ and $w(j) = h_{i}({s_{i}(j)})$ for all $j \geq 1$. Given $w$, we can find an accepted run of $\lb_{i}$, denoted by $q_{i}(1)q_{i}(2)\ldots$, which generates $w$. According to (\ref{eqn:hatT}) and Def. \ref{definition:PA}, there must exist a run $r_{E_i} = (Start, q_{i}(1))(s_{i}(1), q_{i}(2))\ldots$, which is accepted by $E_i$ and generate word $w$. Hence we have $w \in  \mathcal{L} (E_i)$.

Similarly, given a finite word $w^{fin} \in \la(\lb_{i})$ with length $k$ and a trajectory $r_{\lt_i}= s_{i}(1)s_{i}(2)\ldots s_{i}(k)$ of $\lt_i$ generating $w^{fin}$, then we have $s_{i}(j) \rightarrow_i s_{i}(j+1)$ and $w(j) = h_{i}({s_{i}(j)})$ for all $1 \leq j \leq k$. Given $w^{fin}$, we can find an accepted run of $\lb_{i}$, denoted by $q_{i}(1)q_{i}(2)\ldots q_{i}(k+1)$, which generates $w$.
According to (\ref{eqn:hatT}) and Def. \ref{definition:PA}, there must exist a run $r_{E_i} = (Start, q(1))(s_{i}(1), q(2))\ldots(s_{i}(k), q(k+1))$, which is accepted by $E_i$ and generate word $w^{fin}$. Hence we have $w^{fin} \in  \mathcal{L} (E_i)$.
\end{proof}

\subsection{Implementable Global  Behaviors}\label{sec:sub:product}

To solve Prob. \ref{problem:main}, we need to select a word $w$ satisfying the (trace-closed) global specification and also guarantee that $w_{i} = w \upharpoonright_{\Sigma_{i}}$ is executable for all the agents $i\in I$.
Such a word can be obtained from the intersection of the global specification and the implementable global behaviors of the team, which can be modeled by the synchronous product of the implementable local \BAs\ $E_i$.

\begin{definition}[\cite {thiaTech}]\label{def:ProdBA}
The synchronous product  of $n$ mixed \BAs \  $E_i = (Q_{E_i},Q^{in}_{E_i},\Sigma_{E_{i}}, \delta_{E_i}, F_{E_i})$, denoted by $\parallel^{n}_{i=1} E_i$, is an automaton $\lp = (Q_{\lp},Q_{\lp}^{in},\Sigma_{\lp},\delta_{\lp})$, consisting of
\begin{itemize}
\item  a set of states $Q_{\lp}=Q_{E_1}\times \ldots\times Q_{E_n}$;
\item a set of initial states  $Q^{in}_{\lp}=  Q^{in}_{E_1}\times \ldots\times Q^{in}_{E_n}$;
\item a set of inputs $\Sigma_{\lp}= \cup^{n}_{i=1} \Sigma_{E_{i}}$;
\item a transition function $\delta_{\lp}: Q_{\lp} \times \Sigma_{\lp} \rightarrow 2^{Q_{\lp}}$ defined as
$q' \in \delta_{\lp}(q, \sigma)$ such that if $i \in I_\sigma$, $q'[i] \in \delta_{i}(q[i] , \sigma)$, otherwise $q[i] = q'[i]$, where $I_\sigma = \{i \in \{1,\ldots,n\} \mid \sigma \in \Sigma_i \}$ and $q[i]$ denotes the $i$th component of $q$.
\end{itemize}
\end{definition}

The synchronous product composes $n$ components, each of which represents the implementable local specification $E_{i}$ for agent $i$. The synchronous product captures the synchronization among the agents as well as their parallel executions. Informally,
a word $w$ is accepted by $\lp$ if and only if  for each $i \in I$, $w \!\upharpoonright_{\Sigma_{i}}$ is accepted by the corresponding component $E_{i}$. 
A method to find an accepted word of $\lp$ is given in \cite{thiaTech}. 
The next proposition shows that $\la(\lp)$ captures all possible global words that can be implemented by the team.
\begin{proposition}[\cite {thiaTech}]\label{prop:sync}
The language of $\lp$, where $\lp = \parallel_{i\in I} E_i$, is equal to the product of the languages of $E_{i}$ (\ie $\parallel_{i\in I} \la(E_i)$).
\end{proposition}

Finally, we can produce the solution to Prob. \ref{problem:main} by selecting an arbitrary word $w$ from $\la(\lp) \cap \la(\lb_{\phi})$, obtaining the local word $w_{i} = w\upharpoonright_{\Sigma_i}$ and generating the corresponding cc-strategy $r^{c}_{i}$ for each agent.
To find $w\in \la(\lp) \cap \la(\lb_{\phi})$, we can construct an automaton to accept $\la(\lp) \cap \la(\lb_{\phi})$ because of the following proposition:
\begin{proposition}[\cite {thiaTech}]
Let $\lp^1$ and $\lp^2$ be two synchronous products of mixed \BAs. Then a synchronous product $\lp^3$ can be effectively constructed such that $\la(\lp^3) = \la(\lp^1) \cap \la(\lp^2)$.
\end{proposition}

Specifically, we treat $\lb_{\phi}$ as a synchronous product with one component that includes only infinitary accepting states. The overall approach is summarized in Alg. \ref{alg:overall}.
The following theorem shows that the output of Alg. \ref{alg:overall} is indeed the solution to  Prob. \ref{problem:main}.

\begin{algorithm}
\caption{: Synthesis of a set of cc-strategies for a team of agents from a global specification}
\label{alg:overall}
\begin{algorithmic}[1]
\INPUT{A LTL formula $\phi$ over $\Sigma$, a distribution $\{\Sigma_i\subseteq\Sigma,i\in I\}$, and a set of transition systems $\{\lt_i, i\in I\}$}
\OUTPUT{A set of cc-strategies $\{r^c_{i}, i\in I\}$}
\STATE Convert $\phi$ to a \BA\ $\lb_{\phi}$ using LTL2BA \cite{gastin2001fast}

\STATE Run Alg. \ref{alg:trace}
\IF{$\la (\lb_{\phi})$ is not trace-closed}
\RETURN  solution not found
\ELSE
\STATE Construct $\lb_i$ using Alg. \ref{alg:projection} for each $i\in I$
\STATE Construct $E_i = \lt_{i} \otimes \lb_{i}$ (Def. \ref{definition:PA}) for each $i\in I$
\STATE Construct $\lp$ (Def. \ref{def:ProdBA}) and then construct a synchronous product accepting $\la(\lp) \cap \la(\lb_{\phi})$
\IF{$\la(\lp) \cap \la(\lb_{\phi}) = \emptyset$}
\RETURN  solution not found
\ELSE
\STATE Obtain $w \in\la(\lp) \cap \la(\lb_{\phi})$
 \STATE Obtain a set of local words $\{w_i = w\upharpoonright_{\Sigma_i}, i\in I\}$
\STATE Construct a set of automata $\{\lb^c_i, i\in I\}$, each of which accepts only the word $w_i$.
\STATE Construct $E^{c}_i = \lb^c_i \otimes \lt_i$ (Def. \ref{definition:PA}) for all $i\in I$
\STATE Find an accepted run $r_i$ of $E^{c}_i$ and project $r_i$ on $\lt_{i}$ to obtain $r^{c}_{i}$ for all $i\in I$.
\RETURN$\{r^c_i, i\in I\}$
\ENDIF
\ENDIF
\end{algorithmic}
\end{algorithm}

\begin{theorem}
If $\la(\lb_{\phi})$ is trace-closed, the set of cc-strategies $\{r^c_i, \in I\}$ obtained by Alg. \ref{alg:overall} satisfies $\parallel_{i\in I} \{w_{i}\} \neq \emptyset$ and $\parallel_{i\in I} \{w_{i}\} \subseteq \la(\lb_{\phi})$,  where $w_{i}$ is the corresponding word of $\lt_i$  generated by $r^c_i$.
 \end{theorem}
 \begin{proof} Since $w \in \la(\lb_{\phi})$ and $\la(\lb_{\phi})$ is trace-closed, according to Prop: \ref{prop:trace->solution}, we have $\parallel_{i\in I} \{w_{i}\} \subseteq \la(\lb_{\phi})$.  Since $w \in \parallel_{i\in I} \{w_{i}\}$, we have $ \parallel_{i\in I} \{w_{i}\} \neq \emptyset$.
Since $w \in \la(\lp)$ and $ \la(\lp) = \parallel_{i\in I} \la(E_i)$, we have $w\in \parallel_{i\in I} \la(E_i)$. According to Def. \ref{definition:prod_lang}, $w_i \in \la(E_i)$. According to Prop. \ref{prop:runofP}, there exists a trajectory of $\lt_i$ $r^c_i$ generating $w_{i}$, for all $i\in I$.
\end{proof}

\begin{remark}[Completeness]
In the case that $\mathcal L(\lb_{\phi})$ is trace-closed, our approach is complete in the sense that we find a solution to Prob. \ref{problem:main} if one exists.  This follows directly from Prop. \ref{prop:runofP} and the fact that $\la(\lp) = \parallel_{i\in I} \la(E_i)$.  If $\la(\lb_{\phi})$ is not trace-closed,  a complete solution to Prob. \ref{problem:main} requires one to find a non-empty trace-closed subset of $\la(\lb_{\phi})$ if one exists. This problem is not considered in this paper. Therefore, our overall approach to Prob. \ref{problem:main} is not complete.
\end{remark}

\begin{remark}[Complexity]
From a computational complexity point of view, the bottlenecks of the presented approach are the computations relating to $\lp$, because $|Q_{\lp}|$ is bounded above by $\prod_{i\in I}|Q_{E_i}|$ and the upper bound of $|Q_{E_i}|$ is $O(|Q| \cdot |S_{i}|)$.
For most robotic applications, the size of the task specification (\ie $|Q|$) is usually much smaller comparing to the size of the agent model (\ie $|S_{i}|$). Therefore, if we can shrink the size of $Q_{E_i}$ by removing the information about the agent model from $E_{i}$, we can reduce the complexity significantly.
Such reduction can be achieved by using LTL without the next operator and taking a stutter closure of $E_{i}$. This will be addressed in our future work.
\end{remark}

\section{Automatic Deployment in RULE}
\label{sec:casestudy}

In this section, we show how our method can be used to deploy a team of Khepera III car-like robots in our Robotic Urban-Like Environment (Fig. \ref{fig:platform}). The platform consists of a collection of roads, intersections, and parking lots. Each intersection has traffic lights. The city is easily reconfigurable by re-taping the platform. All the cars can communicate through Wi-Fi with a desktop computer, which is used as an interface to the user ({\it i.e.,} to enter the global specification) and to perform all the computation necessary to generate the individual cc-strategies.
Once computed, these are sent to the cars,
which execute the task autonomously by interacting with the environment and by communicating with each other, if necessary. We assume that inter-robot communication is always possible.

 We model the motion of each robot in the platform using a transition system, as shown in Fig. \ref{RULE_TS}.
 The set of states $S_{i}$ is the set of labels assigned to roads, intersections and parking lots (see Fig. \ref{fig:platform}) and the relation $\rightarrow_{i}$ shows how these are connected.
 We distinguish one bound of a road from the other since the parking lots can only be located on one side of each road. For example, we use $R_{1r}$ and $R_{1l}$ to denote the two bounds  of road $R_{1}$. Each state of $\lt_i$ is associated with a set of motion primitives. For example, at region $R_{1r}$, which corresponds to the access point for parking lot $P_1$
(see Fig. \ref{RULE_TS}), the robot can choose between two motion primitives: \texttt{follow\_road} and {\tt park}, which allow the robot to stay on the road or turn right into $P_{1}$.
If the robot follows the road, it reaches the vertex $I_{2}$, where three motion primitives are available:   {\tt U\_turn}, {\tt turn\_right\_int}, and {\tt go\_straight\_int}, which allow the robot to make a U-turn, turn right or go straight through the intersection.  It can be seen that, by selecting a motion primitive available at a region, the robot can correctly
execute a trajectory of $\lt_i$, given that it is initialized at a vertex of $\lt_{i}$.  The choice of a motion primitive uniquely determines the next vertex. In other words, a set of cc-strategies defined in Sec. \ref{sec:prob_form} and obtained as described in Sec. \ref{sec:synthesis} can be immediately implemented by the team.

\begin{figure}
\center
   \begin{tabular}{cc}
         \includegraphics[scale = 0.8]{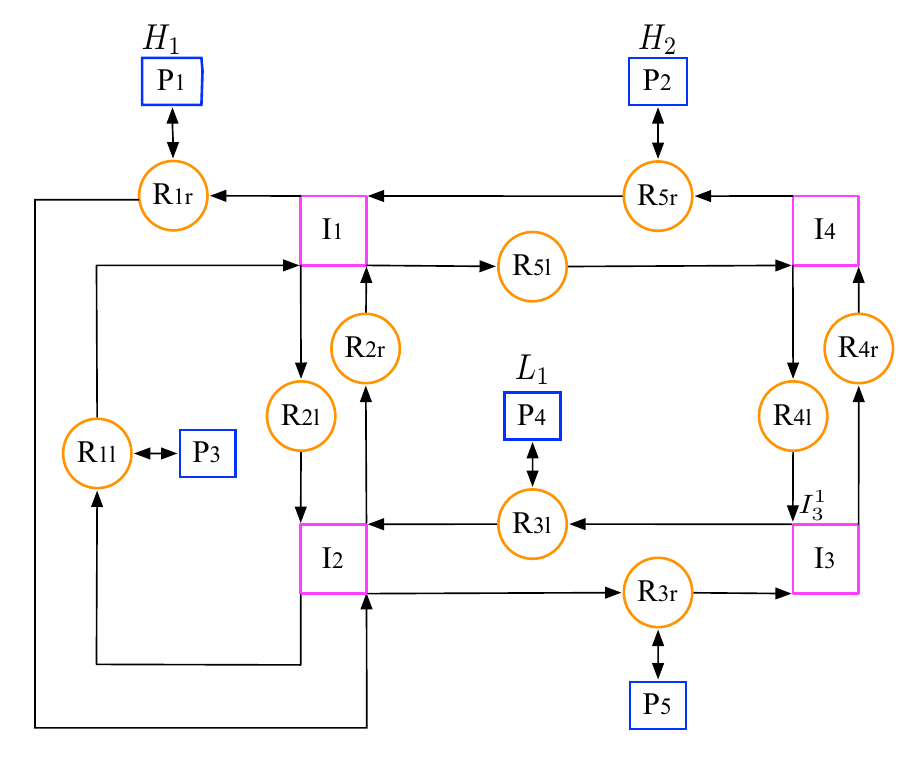}
   \end{tabular}
\caption{Transition system $\lt_1$ for robot $1$. The states represents the vertices in the environmental graph (Fig. \ref{fig:platform}), $s_{0_{1}}$ shows that robot $1$ starts at $R_{1r}$; $\rightarrow_{1}$ captures the connectivity between the vertices;  $h_{1}$ captures the locations of  the unsafe regions and the requests. The dummy request $\varpi_1$ is assigned to all the vertices that have no property and is omitted in this figure.
} \label{RULE_TS}
\end{figure}

Assume that service requests, denoted by $H_1, H_2, L_1, L_2$ and $L_3$, occur at parking lots $P_1, P_2, P_4, P_5$ and $P_3$, respectively.  ``H" stands for ``heavy'' requests requiring the efforts of multiple cars while ``L"  represents  ``light'' requests that only need one car to service. Specifically,
$H_1$ is shared by all three cars and $H_2$ is shared between car $1$ and $2$.
As we can see in Fig. \ref{fig:platform}, the number of parking spaces of a parking lot equals the number of cars needed to service the request that occurs at this parking lot.  For example, $P_1$ where $H_1$ occurs has three parking spaces.
Besides the set of requests, we also consider some regions to be unsafe. In this example, we assume that intersection $I_3$ is unsafe for all robots before request $H_1$ is serviced.  We use the output map $h_{i}$ of $\lt_{i}$ (see Fig. \ref{RULE_TS}) to capture the locations of requests and unsafe regions. A ``dummy request'' $\varpi_{i}$ is assigned to all the other regions. We use a special semantics for $\varpi_{i}$: a robot does not service any request when visiting a region where $\varpi_{i}$ occurs.

 We model the capabilities of the cars to service requests while considering unsafe regions as a distribution: $\Sigma_1 = \{H_1, H_2, L_1, I_3^1, \varpi_1\}, \Sigma_2 = \{H_1, H_2, L_2, I_3^2, \varpi_2\}$ and  $\Sigma_3 = \{H_1,  L_3, I_3^3, \varpi_3\}$. Note that we treat the unsafe region $I_{3}$ as an independent property assigned to each car since it does not require the cooperation of the cars.
We aim to find a satisfying set of individual cc-strategies for each robot to satisfy the global specification $\phi$, which is the conjunction of the following LTL formulas over the set of properties $\Sigma = \{H_1, H_2, L_1, L_2, L_3, I^1_3, I^2_3, I^3_3, \varpi_1, \varpi_2, \varpi_3\}$:
\begin{enumerate}
\item Request $H_2$ is serviced infinitely often.
$$
\gl \ev H_2
$$
\item First service request $H_1$, then service request $L_1$ and $L_2$ regardless of the order or request $L_3$.
$$
\ev (H_1 \andltl (\ev(L_1 \andltl L_2) \orltl \ev L_3))
$$
\item Do not visit intersection $I_3$ until $H_1$ is serviced.
$$
\notltl(I^1_3 \orltl I^2_3 \orltl I^3_3)\  \un \ H_1
$$
\end{enumerate}

By applying Alg. \ref{alg:overall}, we first learn that the language satisfying ${\phi}$ is trace-closed. Then, we obtain the implementable automaton $\parallel_{i\in I}E_i$ as described in Sec. \ref{sec:sub:imple} and \ref{sec:sub:product}.  Finally, we choose a word $w \in \mathcal L (\lb_{\phi}) \cap  \mathcal L(\parallel_{i\in I} E_i)$ and project $w$ on the local alphabets $\Sigma_{i}$, $i\in \{1,2,3\}$ to obtain the local words, which lead to the following cc-strategies:

\begin{small}
$$
\begin{array}{l c l}
r^c_1 & = & R_{1r}I_2R_{2r}I_1R_{1r}P_1R_{1r}I_2R_{3r}I_3R_{4r}I_4R_{5r}P_2P_2\ldots,\\
r^c_2 & = & R_{5r}I_4R_{1r}P_1R_{1r}I_2R_{2r}I_1R_{5l}I_4R_{5r}P_2P_2\ldots,\\
r^c_3 & = & R_{2r}I_1R_{1r}P_1R_{1r}I_2R_{1l}P_3.
\end{array}
$$
\end{small}The language satisfying the global specification $\phi$ includes only infinite words. Hence, both cars $1$ and $2$ have infinite cc-strategies, since $H_{2}$ needs to be serviced infinitely many times. Note that car $3$ has a finite cc-strategy. The synchronization is only triggered when the cars are about to service shared requests, \ie when at $P_{1}$ and $P_{2}$. Besides these synchronization  moments, the cars follow their cc-strategies and execute their individual tasks in parallel, which speed up the process of accomplishing the global task. Snapshots from a movie of the actual deployment are shown in Fig. \ref{fig:snapshots}.
The movie of the deployment in the RULE platform is available at \url{http://hyness.bu.edu/CDC2011}.

\begin{figure}
\begin{tabular}{c@{\ }c@{\ }c}
\includegraphics[scale=0.225]{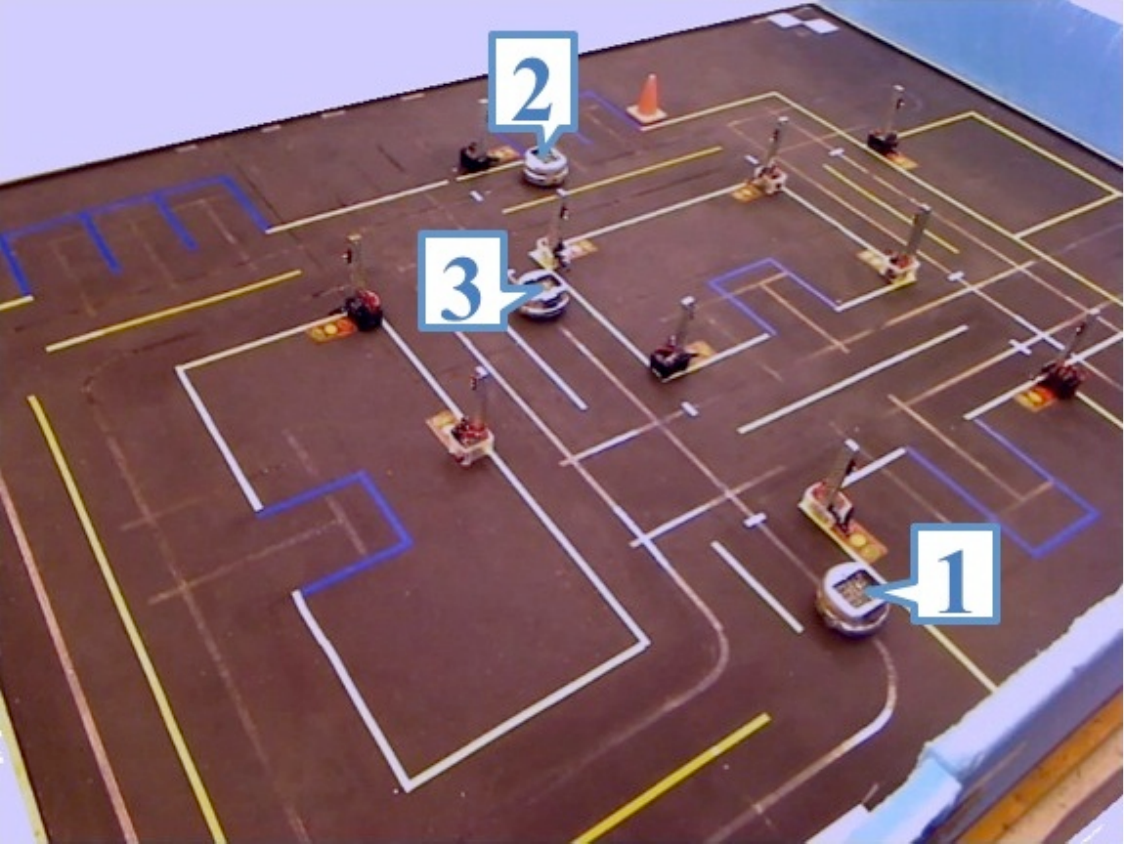} & 
\includegraphics[scale=0.17]{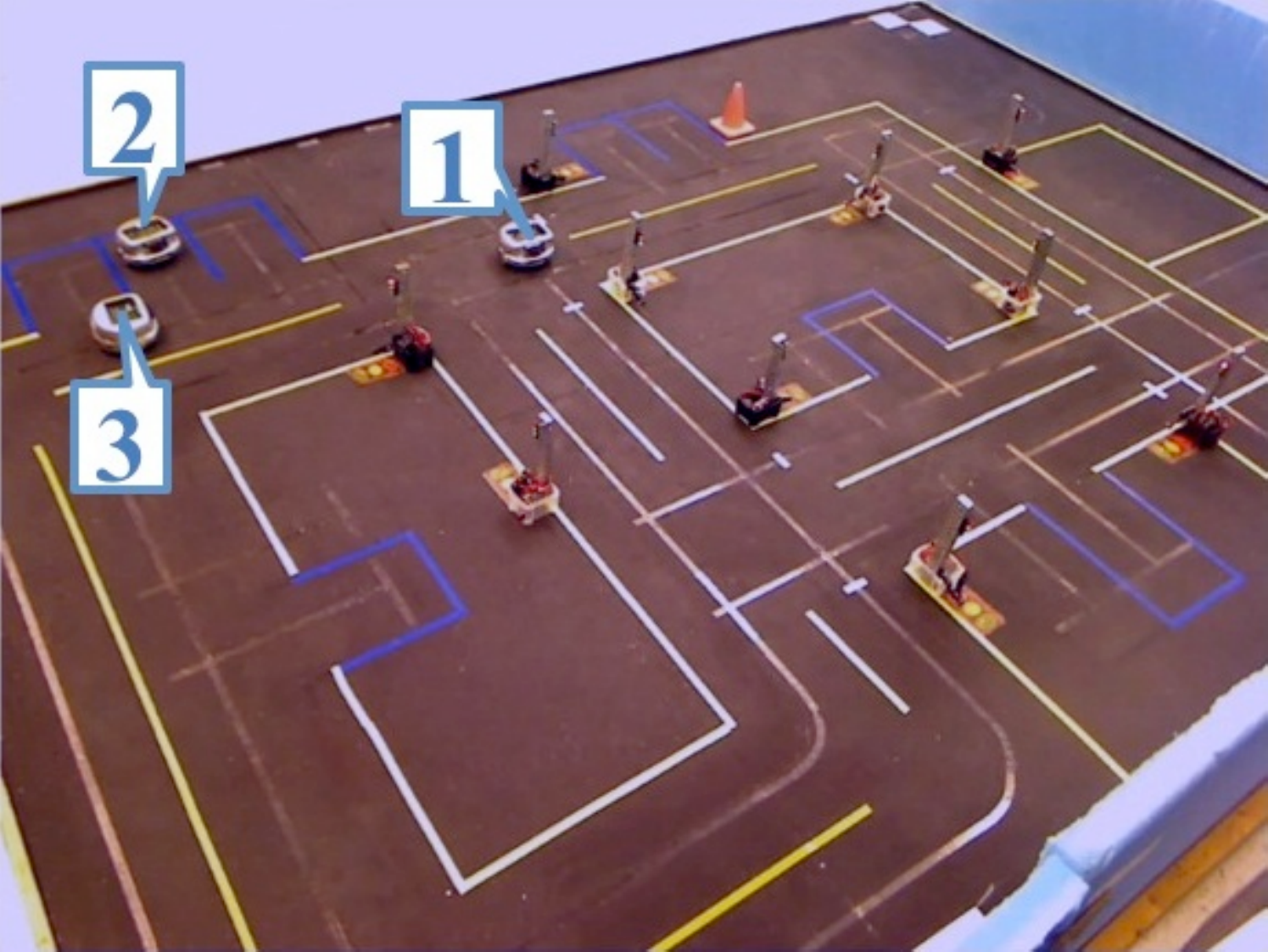} &  
\includegraphics[scale=0.17]{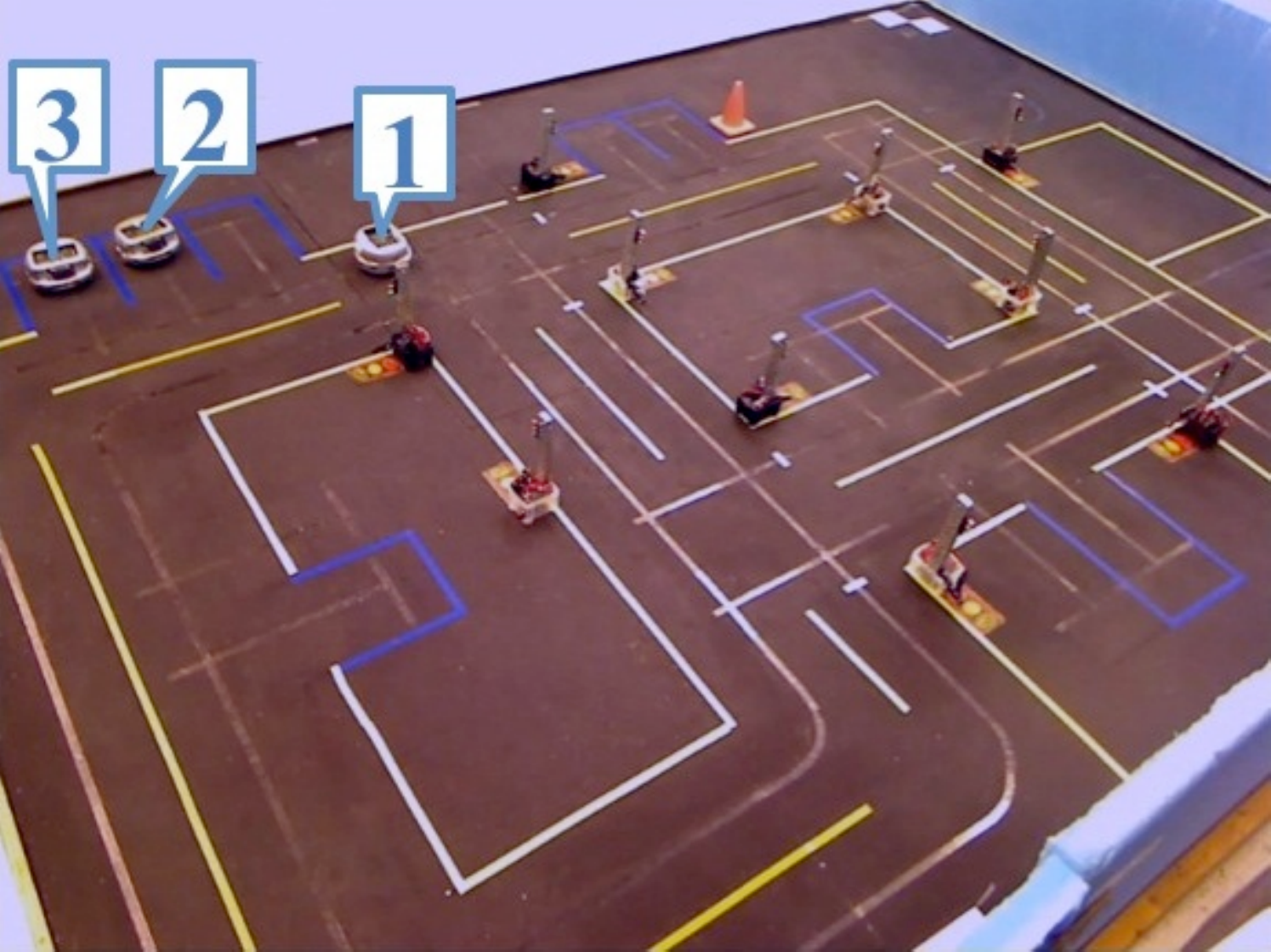}\\
(1) & (2) & (3) \\
\includegraphics[scale=0.17]{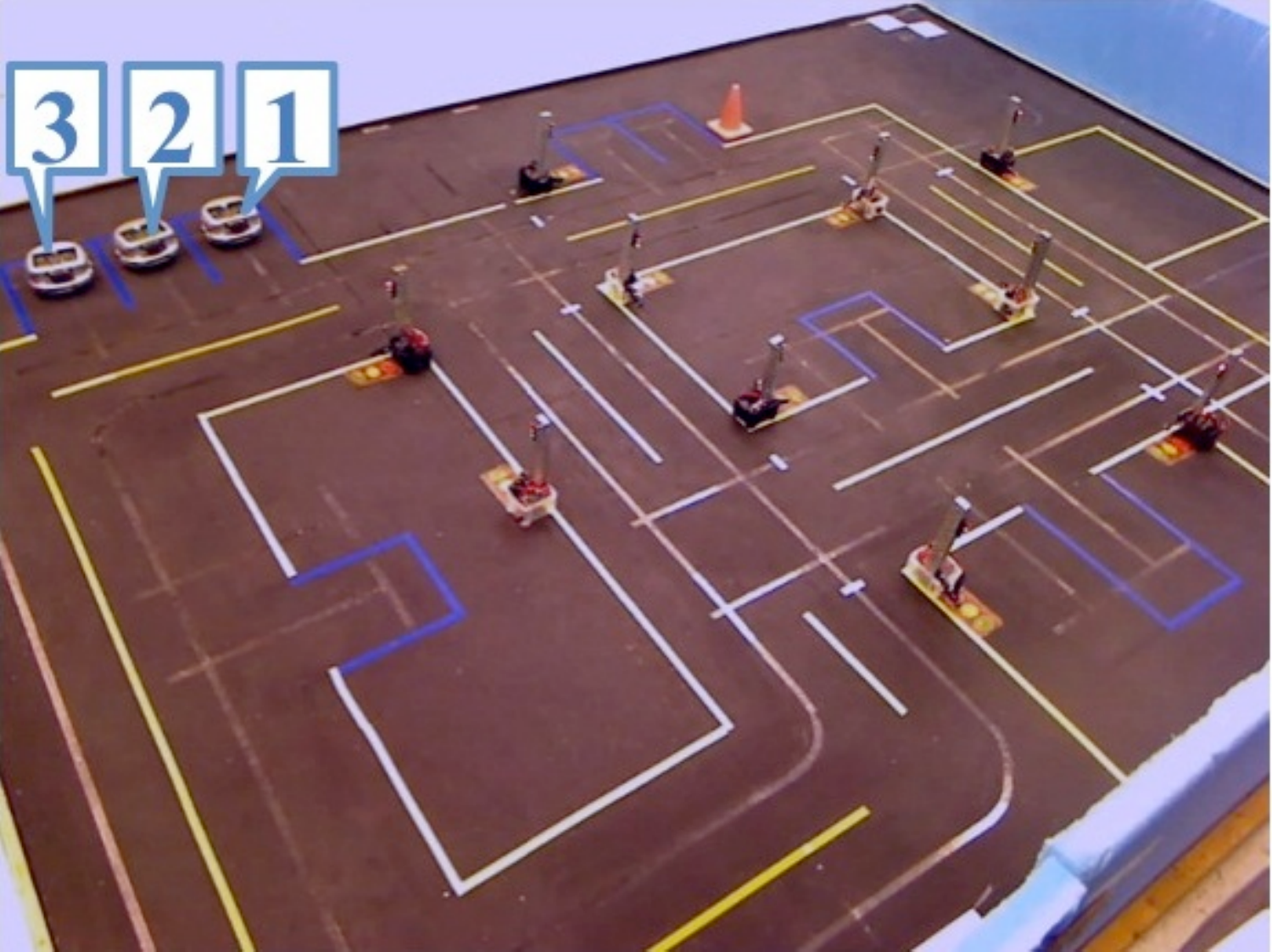}  & 
\includegraphics[scale=0.17]{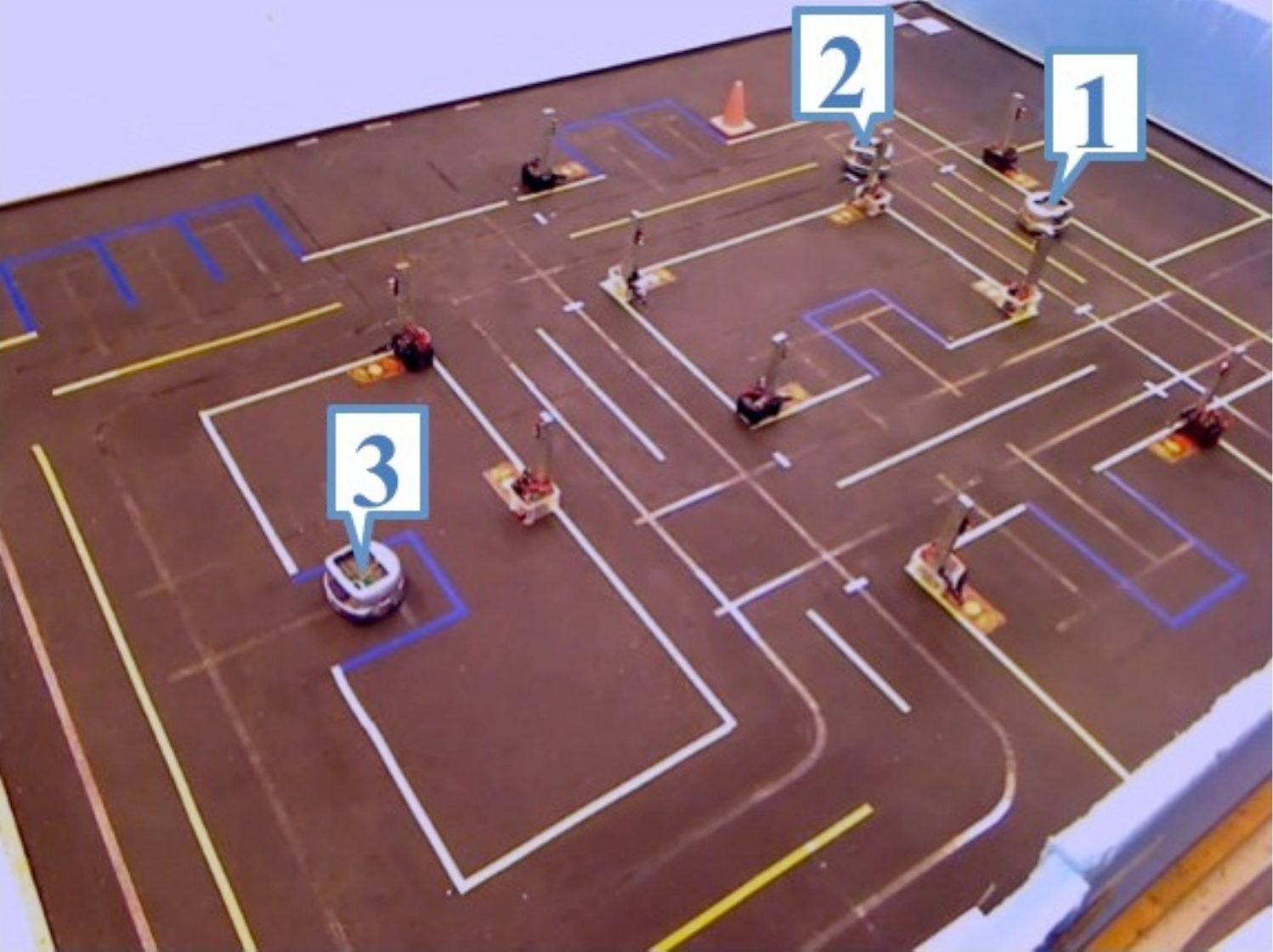}  &
\includegraphics[scale=0.17]{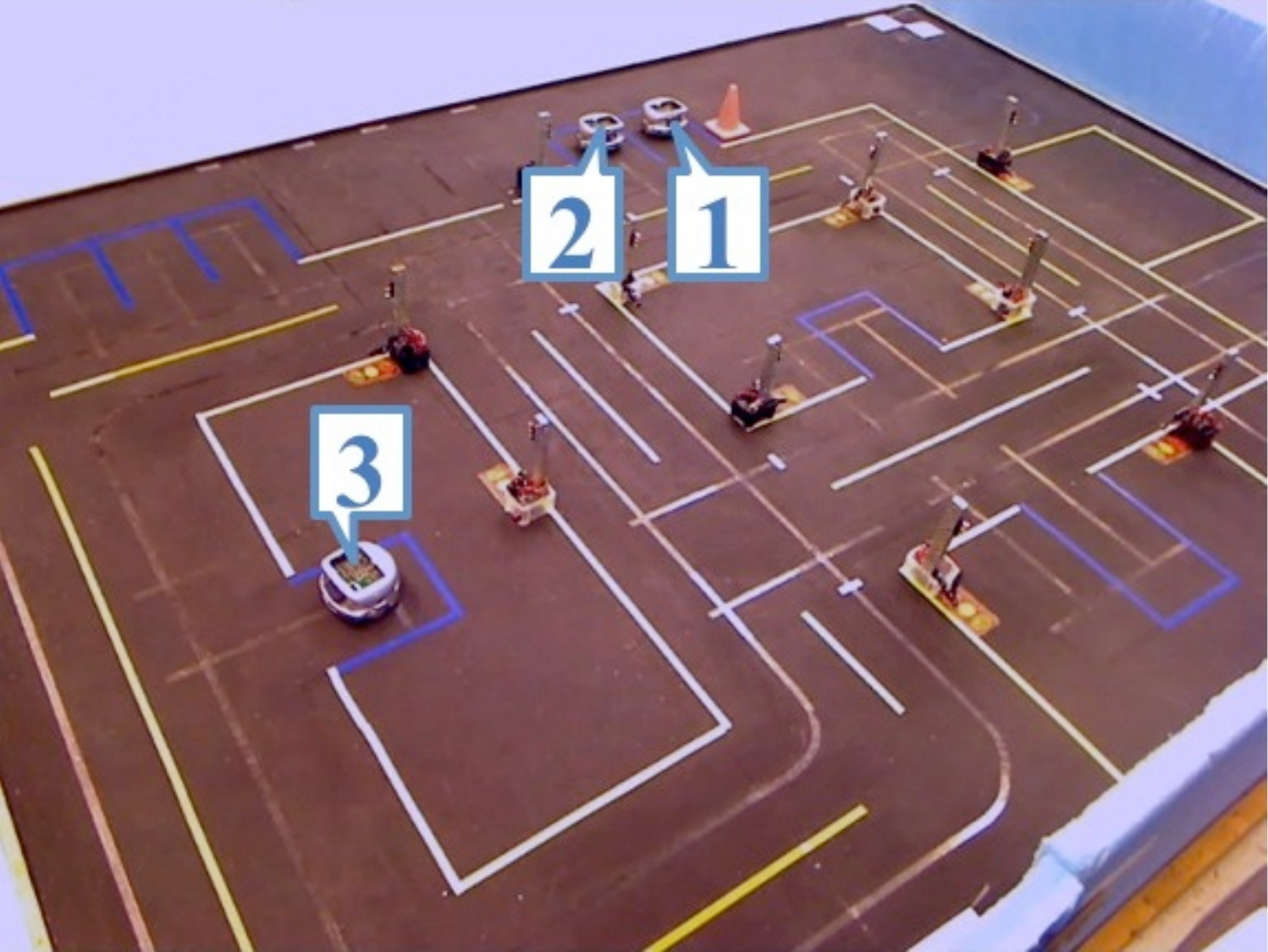} \\
(4) & (5) & (6)
\end{tabular}
\caption{Six snapshots from the deployment corresponding to the given cc-strategies.
The labels for the roads, intersections, and parking spaces are given in Fig. \ref{fig:platform}.
(1) the position of the cars immediately after the initial time, when robots $1$, $2$ and $3$ are on roads $R_{1r}, R_{5r}$ and $R_{2r}$, respectively; (2) robot $2$ is waiting for the other two robots to enter parking lot $P_1$ at which the heavy request $H_1$ occurs; (3) both robots $2$ and $3$ are at $P_{1}$ waiting for robot $1$; (4) all three robots are at $P_{1}$ simultaneously, and therefore request $H_1$ is serviced; (5) robot $3$ services the light request $L_{3}$ at $P_{3}$ and finishes its task; (6) eventually robots $1$ and $2$ stop at $P_2$ and service $H_{2}$ together infinitely many times.}\label{fig:snapshots}
\end{figure}

\section{Conclusions and Future Works}\label{sec:future}

We present an algorithmic framework to deploy a team of agents from a task specification given as an LTL formula over a set of properties. Given the agent capabilities to satisfy the properties, and the possible cooperation requirements for the shared properties, we find individual control and communication strategies such that the global behavior of the system satisfies the given specification. We illustrate the proposed method with experimental results in our Robotic Urban-Like Environment (RULE).

As future work, we will consider reducing the computational complexity and applying this approach to a team of agents with continuous dynamics.
Also, we plan to accommodate more realistic models of agents that can capture uncertainty and noise in the system, such as Markov Decision Processes(MDP) and Partially Observed Markov Decision Processes(POMDP), and probabilistic specification languages such as PLTL.



\bibliographystyle{IEEEtran}
\bibliography{refer,}
\end{document}